\newif\ifIcassp

\Icassptrue

\ifIcassp
    \documentclass{article}
    \usepackage{spconf}
\else 
    \documentclass[journal,12pt,onecolumn,draftclsnofoot]{IEEEtran}
\fi 

\usepackage{comment}

\usepackage{algorithm}
\usepackage{algorithmic}
\usepackage{graphicx}
\usepackage{url}
\usepackage{breqn}
\usepackage{cite}
\usepackage[usenames]{color}
\usepackage{amsfonts}
\usepackage{fancyhdr}
\usepackage{todonotes}
\usepackage{arydshln}
\usepackage{amsmath,amssymb,amsthm,bm,mathtools}
\usepackage{psfrag, caption,subcaption}
\newtheorem{theorem}{Theorem}
\newtheorem{Proposition}{Proposition}
\newtheorem{lemma}{Lemma}
\newtheorem{Definition}{Definition}

\usepackage{arydshln}
\usepackage{booktabs}
\usepackage{multirow}
\usepackage{adjustbox}
\usepackage{hyperref}
\usepackage{tikz}
\usetikzlibrary{
    arrows.meta, calc, positioning, backgrounds,
    decorations.markings, decorations.pathreplacing
}
\usepackage[most]{tcolorbox}

\usepackage{acronym,comment}

\acrodef{dag}[DAG]{Directed acyclic graph}
\acrodef{iff}[\emph{iff}]{if and only if}
\acrodef{sem}[SEM]{structural equation model}

\newcommand{\R}{\mathbb{R}}

\newcommand{\transp}{{\mathsf T}}

\newcommand{\vecx}[1]{\bm{#1}}
\newcommand{\matx}[1]{\bm{#1}}
\newcommand{\setx}[1]{\mathcal{#1}}

\newcommand{\downset}{\mathord{\downarrow}}
\newcommand{\Atimes}{\vecx A_{\times}}
\newcommand{\Wtimes}{\vecx W_{\times}}

\newtcolorbox{worked}{colback=black!3,colframe=black!55,boxrule=0.5pt,
  left=4pt,right=4pt,top=3pt,bottom=3pt,fonttitle=\bfseries\small,
  title={Example 1: Product DAG Causal Shift Operator}}

\newtcolorbox{worked1}{colback=black!3,colframe=black!55,boxrule=0.5pt,
  left=4pt,right=4pt,top=3pt,bottom=3pt,fonttitle=\bfseries\small,
  title={Example 2: Separable Product DAG Filter}}

\usepackage{hyperref}

\def\cast{{
   \mathord{
      \hbox to 0em{
         \ooalign{
	   \smash{\hbox{$\ast$}}\crcr
	   \smash{\hskip-1pt\Large\hbox{$\circ$}} }
	 \hidewidth}
      \phantom{\bigcirc}
} }}

\def\bm#1{\mbox{\boldmath $#1$}}

\newcommand{\bds}{\begin {itemize}}
\newcommand{\eds}{\end {itemize}}
\newcommand{\bdf}{\begin{definition}}
\newcommand{\blm}{\begin{lemma}}
\newcommand{\edf}{\end{definition}}
\newcommand{\elm}{\end{lemma}}
\newcommand{\bthm}{\begin{theorem}}
\newcommand{\ethm}{\end{theorem}}
\newcommand{\bprp}{\begin{prop}}
\newcommand{\eprp}{\end{prop}}
\newcommand{\bcl}{\begin{claim}}
\newcommand{\ecl}{\end{claim}}
\newcommand{\bcr}{\begin{coro}}
\newcommand{\ecr}{\end{coro}}
\newcommand{\bquest}{\begin{question}}
\newcommand{\equest}{\end{question}}

\newcommand{\larrow}{{\larrow}}

\newcommand{\argmin}{\ensuremath{\mathrm{arg}\min}}
\newcommand{\argmax}{\ensuremath{\mathrm{arg}\max}}

\usepackage{latexsym}

\def\IC{\mathbb C}
\def\IN{\mathbb N}
\def\IZ{\mathbb Z}
\def\IR{\mathbb R}

\def\shat{^{\mathchoice{}{}%
 {\,\,\smash{\hbox{\lower4pt\hbox{$\widehat{\null}$}}}}%
 {\,\smash{\hbox{\lower3pt\hbox{$\hat{\null}$}}}}}}

\def\bSigma{{
      \ooalign{
      \smash{\hskip.4pt\raise.4pt\hbox{$\Sigma$}}\vphantom{}\crcr
      \smash{\hskip.7pt\raise.6pt\hbox{$\Sigma$}}\vphantom{}\crcr
      \smash{\hbox{$\Sigma$}}\vphantom{$\Sigma$}}
      \vphantom{\hbox{$\Sigma$}}
      }}
\def\bTheta{{
      \ooalign{
      \smash{\hskip.5pt\raise.5pt\hbox{$\Theta$}}\vphantom{}\crcr
      \smash{\hskip.0pt\raise.1pt\hbox{$\Theta$}}\vphantom{}\crcr
      \smash{\hbox{$\Theta$}}\vphantom{$\Theta$}}
      \vphantom{\hbox{$\Theta$}}
      }}
\def\bDelta{{
      \ooalign{
      \smash{\hskip.4pt\raise.4pt\hbox{$\Delta$}}\vphantom{}\crcr
      \smash{\hskip.7pt\raise.6pt\hbox{$\Delta$}}\vphantom{}\crcr
      \smash{\hbox{$\Delta$}}\vphantom{$\Delta$}}
      \vphantom{\hbox{$\Delta$}}
      }}
\def\bLambda{{
      \ooalign{
      \smash{\hskip.5pt\raise.5pt\hbox{$\Lambda$}}\vphantom{}\crcr
      \smash{\hskip.0pt\raise.1pt\hbox{$\Lambda$}}\vphantom{}\crcr
      \smash{\hbox{$\Lambda$}}\vphantom{$\Lambda$}}
      \vphantom{\hbox{$\Lambda$}}
      }}

\makeatletter

\def\bordermatrix#1{\begingroup \m@th
  \@tempdima 8.75\p@
  \setbox\z@\vbox{%
    \def\cr{\crcr\noalign{\kern2\p@\global\let\cr\endline}}%
    \ialign{$##$\hfil\kern2\p@\kern\@tempdima&\thinspace\hfil$##$\hfil
      &&\quad\hfil$##$\hfil\crcr
      \omit\strut\hfil\crcr\noalign{\kern-\baselineskip}%
      #1\crcr\omit\strut\cr}}%
  \setbox\tw@\vbox{\unvcopy\z@\global\setbox\@ne\lastbox}%
  \setbox\tw@\hbox{\unhbox\@ne\unskip\global\setbox\@ne\lastbox}%
  \setbox\tw@\hbox{$\kern\wd\@ne\kern-\@tempdima\left[\kern-\wd\@ne
    \global\setbox\@ne\vbox{\box\@ne\kern2\p@}%
    \vcenter{\kern-\ht\@ne\unvbox\z@\kern-\baselineskip}\,\right]$}%
  \null\;\vbox{\kern\ht\@ne\box\tw@}\endgroup}
\makeatother

\makeatletter
\def\argmin{\mathop{\operator@font arg\,min}}
\def\argmax{\mathop{\operator@font arg\,max}}
\makeatother

\newcommand{\diag}{\mbox{\rm diag}}

\newcommand{\rank}{\mbox{\rm rank}}

\newcommand{\vect}{\mbox{\rm vec}}

\def\bm#1{\mbox{\boldmath $#1$}}

\newcommand{\bea}{\begin{array}}
\newcommand{\ena}{\end{array}}
\newcommand{\beq}{\begin{equation}}
\newcommand{\enq}{\end{equation}}

\newcommand{\beqa}{\begin{eqnarray}}
\newcommand{\enqa}{\end{eqnarray}}

\newcommand{\beqan}{\begin{eqnarray*}}
\newcommand{\enqan}{\end{eqnarray*}}

\newcommand{\AL}{\begin{enumerate}}
\newcommand{\ALE}{\end{enumerate}}

\def\addots{\mathinner{
    \mkern1mu\raise0pt\vbox{\kern7pt\hbox{.}}
    \mkern2mu\raise4pt\hbox{.}
    \mkern2mu\raise7pt\hbox{.}
    \mkern1mu}}

\def\sddots{\mathinner{
    \mkern.8mu\raise7pt\hbox{.}
    \mkern.8mu\raise4pt\hbox{.}
    \mkern.8mu\raise0pt\vbox{\kern7pt\hbox{.}}
    \mkern1mu}}

\def\saddots{\mathinner{
    \mkern.2mu\raise0pt\vbox{\kern7pt\hbox{.}}
    \mkern.2mu\raise4pt\hbox{.}
    \mkern.2mu\raise7pt\hbox{.}
    \mkern1mu}}

\def\sqplus{\mathbin{
	{\ooalign{\hfil\raise.3ex\hbox{\scriptsize
	+}\hfil\crcr\mathhexbox274\crcr\mathhexbox275}}
	}} 
\def\sqminus{\mathbin{
	{\ooalign{\hfil\raise.3ex\hbox{\scriptsize
	--}\hfil\crcr\mathhexbox274\crcr\mathhexbox275}}
	}}

\def\IC{{
   \mathord{
      \hbox to 0em{
	 \hskip-4pt
         \ooalign{
	   \smash{\hskip1.9pt\raise2.6pt\hbox{$\scriptscriptstyle |$}}\crcr
	   \smash{\hbox{\rm\sf C}} }
	 \hidewidth}
      \phantom{\hbox{\rm\sf C}}
} }}
\def\IN{
    {\ooalign{
   \smash{\hskip2.2pt\raise1.5pt\hbox{$\scriptscriptstyle |$}}\vphantom{}\crcr
   \hbox{\sf N}
	}}
	} 
\def\IZ{
    {\ooalign{
   \smash{\hskip1.9pt\raise0pt\hbox{$\sf Z$}}\vphantom{}\crcr
   \hbox{\sf Z}
	}}
	} 
\def\IR{
    {\ooalign{
   \smash{\hskip2.2pt\raise1.5pt\hbox{$\scriptscriptstyle |$}}\vphantom{}\crcr
   \smash{\hskip2.2pt\raise3.3pt\hbox{$\scriptscriptstyle |$}}\vphantom{}\crcr
   \hbox{\sf R}
	}}
	} 

\DeclareMathAlphabet{\mathcmb}{OT1}{cmr}{b}{n}

\def\bSigma{\ensuremath{\mathcmb{\Sigma}}}
\def\bLambda{\ensuremath{\mathcmb{\Lambda}}}

\def\bTheta{\ensuremath{\mathcmb{\Theta}}}

\newcommand{\SI}{\begin{indlist}}
\newcommand{\EI}{\end{indlist}}

\newcommand{\DL}{\begin{dashlist}}
\newcommand{\DLE}{\end{dashlist}}

\makeatletter
\def\setboxz@h{\setbox\z@\hbox}
\def\wdz@{\wd\z@}
\def\boxz@{\box\z@}
\def\underset#1#2{\binrel@{#2}%
  \binrel@@{\mathop{\kern\z@#2}\limits_{#1}}}
\def\binrel@#1{\begingroup
  \setboxz@h{\thinmuskip0mu
    \medmuskip\m@ne mu\thickmuskip\@ne mu
    \setbox\tw@\hbox{$#1\m@th$}\kern-\wd\tw@
    ${}#1{}\m@th$}%
  \edef\@tempa{\endgroup\let\noexpand\binrel@@
    \ifdim\wdz@<\z@ \mathbin
    \else\ifdim\wdz@>\z@ \mathrel
    \else \relax\fi\fi}%
  \@tempa
}
\let\binrel@@\relax%
\makeatother

\begin{document}

\title{Signal Processing over Product DAG\MakeLowercase{s}: Causal Shifts and Filters}
\ifIcassp
    \ninept
\makeatletter
\g@addto@macro\small{\abovedisplayskip 4pt plus 2pt minus 2pt \belowdisplayskip 4pt plus 2pt minus 2pt \abovedisplayshortskip 0pt plus 2pt \belowdisplayshortskip 2pt plus 2pt}
\makeatother
\normalsize
\setlength{\textfloatsep}{7pt plus 2pt minus 2pt}
\setlength{\floatsep}{6pt plus 2pt minus 2pt}
%

\name{Sundeep Prabhakar Chepuri$^\star$, Antonio G. Marques$^\ddag$, Maulik Devmurari$^\star$, and Gonzalo Mateos$^\dag$  
}
\address{$^\star$Indian Institute of Science, India, 
$^\ddag$King Juan Carlos Univ., Spain, and 
$^\dag$Univ. of Rochester, USA
}
\fi 

\maketitle

\begin{abstract}
We develop a signal processing framework for signals indexed by the
\emph{product} of two directed acyclic graphs (DAGs) and described by a linear structural equation model (SEM). Such a setup arises whenever (linear) causal relations act along two domains, as in component versus manufacturing stage or gene versus experimental condition. 
Disregarding the factorization of the underlying graph and the native two-axis causal structure requires inverting a weighted transitive closure matrix whose size is the product of the two factor sizes for Fourier analysis. Recognizing that standard graph products fail to yield factorizable transitive closures, 
we introduce a new DAG product under which separability holds. We motivate the new operator in the vertex domain, and show that it also renders the SEM, the Fourier modes, the causal shifts, and the filters on the product DAG separable across its constituent graph factors. 
Graph signal denoising tests showcase the benefits of the novel product DAG filters.
\end{abstract}

\ifIcassp
\fi 
\section{Introduction} \label{sec:intro} 
Graph signal processing (GSP)~\cite{shuman2013emerging,leus2023graph} deals with data indexed by the nodes of a graph. Rooted in spectral graph theory, most of the advances in this field rely on matrix representations of the graph (typically the adjacency or Laplacian matrices) whose eigendecomposition induces a meaningful notion of frequency. Beyond Fourier analysis, workhorse tools such as filtering, sampling, or spectral estimation can be generalized via polynomials of those so-termed graph-shift operators (GSOs). However, accounting for edge directionality comes with well-documented challenges~\cite{marques2020spmag}. These compound for \emph{directed acyclic graphs (DAGs)}, which are central to a number of timely applications where the relations among variables are causal rather than associative; see e.g.,~\cite{peters2017elements}. Indeed, because the adjacency matrix of a DAG is nilpotent, it is neither diagonalizable nor endowed with a meaningful spectrum, hence there is no well-defined graph Fourier transform; see e.g.,~\cite{mihal2025tsp}. Recently, a  DAG-specific signal processing (SP) framework was put forth in~\cite{seifert2023causal}, where GSOs and convolutional filters are constructed based on the DAG's weighted transitive closure (Section \ref{sec:prelims}). This way, when signals adhere to a linear \ac{sem} it is possible to offer a fresh Fourier interpretation to causal notions~\cite{misiakos2023neurips,misiakos2024icassp}; see also~\cite{stankovic2025dsp} for a different approach based on graph zero-padding.

Most germane to this paper's innovations are signals supported on \emph{product} graphs, which arise with data indexed by two (or more) domains. Prominent examples include recommender systems and spatio-temporal fields, just to name a couple. Since the number of vertices of the product is $N=N_1N_2$, naive processing algorithms face scalability issues. In addition, when model parameters must be estimated or learned, the sample size requirements grow accordingly. Yet, judicious exploitation of (latent) product graph structure can keep these apparent predicaments in check~\cite{sandryhaila2014bigdata}. The eigenvectors of the product shift are the Kronecker product of the factor eigenvectors, so Fourier analysis decouples across domains. This separability principle underpins time-vertex SP~\cite{grassi2018timevertex} for scalable filtering, graph learning, and sample selection on product
graphs~\cite{sandryhaila2013discrete,lodhi2020product,kadambari2021product,ortiz2019sparse}.

This paper builds on the recognition that such elegant GSP generalizations do \emph{not} carry over to product DAGs. Again, the cardinal reason is the collapsing spectrum of the envisioned DAG adjacency matrix factors. Even more, the weighted transitive closure adopted as Fourier basis for DAGs~\cite{seifert2023causal}  cannot be rendered separable across factors for any of the standard (i.e., Cartesian, Kronecker, or strong \cite{sandryhaila2014bigdata}) graph products. To bridge this gap, our main contribution is a new definition of product DAG (Section \ref{sec:productdags}). We first explain its rationale in the vertex domain, elucidating that it amounts to the strong product with a sign-reversed cross term that avoids spurious duplication of information flowing from nodes' ancestors. 
Crucially, we also establish a resolvent factorization (Proposition \ref{prop:resolvent}) which leads to tractable generalizations of the transitive closure, Fourier modes, and GSOs, all of which decompose as Kronecker products of their factor counterparts. In Section \ref{sec:filters}, we discuss separable convolutional filtering on product DAGs. Graph signal denoising experiments over synthetic and real-world data demonstrate the benefits of the proposed product DAG filters (Section \ref{sec:experiments}). We conclude in Section \ref{sec:conclusions} with a summary of our findings and an outlook towards future work. 
\vspace{2pt}

\noindent\textbf{Notation.} We use lowercase (uppercase) boldface letters to denote vectors (matrices) and calligraphic letters to denote sets.  We denote transpose with $(\cdot)^{\transp}$, Kronecker and Hadamard products as $\otimes$ and $\odot$, respectively. $\vecx 1_N$ denotes the all-ones vector of length $N$, $\matx I_N$ denotes the $N \times N$ identity matrix, and $\vecx e_i$ is the $i$th column of $\matx I_N$. The notation $i\preceq_k q$ means $i=q$ or $i$ reaches $q$, for $i,q$ in some set~$\setx V_k$. 

\section{Preliminaries}\label{sec:prelims}

\noindent\textbf{DAGs and \acp{sem}.} Let $\setx G=(\setx V,\setx E,\vecx A)$ be a weighted DAG with $N$ vertices, vertex
set $\setx V$, edge set $\setx E$, and weighted adjacency matrix $\vecx A$, whose $(v,u)$
entry $A_{vu}$ is nonzero \ac{iff} the parent node $u$ influences the child node $v$.
Since $\setx G$ is acyclic, ordering $\setx V$ topologically renders $\vecx A$ strictly lower triangular, and hence nilpotent.
 
Let $\vecx s \in \R^{N}$ be a DAG signal adhering to a linear \ac{sem}, i.e.,
\begin{equation}
\vecx s=\vecx A\vecx s+\vecx c,
\label{eq:sem}
\end{equation}
where $\vecx c \in \R^{N}$ collects independent exogenous \emph{root causes}. Since $\vecx I-\vecx A$ is invertible, we
can equivalently write $\vecx s=\vecx W\vecx c$, where
\begin{equation}
\vecx W:=(\vecx I-\vecx A)^{-1}=\textstyle\sum\limits_{n=0}^{\infty}\vecx A^{n}=\textstyle\sum\limits_{n=0}^{N-1}\vecx A^{n}
\label{eq:Wmat}
\end{equation}
is the (reflexive) \emph{weighted transitive closure} with the series truncating at $n=N-1$ by
nilpotency. Since $[\vecx A^{n}]_{vu}\neq 0$ means that there is at least one $n$-hop directed path from $u$ to $v$ and the entry $W_{vu}$ sums the products of the edge weights over
all directed paths from $u$ to $v$.\vspace{2pt}

\noindent\textbf{Fourier Modes.} Interestingly, the pair of equations
\begin{equation}\label{eq:an_syn}
\vecx c=(\vecx I-\vecx A) \vecx s \quad\text{and}\quad\vecx s=\vecx W\vecx c
\end{equation}
can be interpreted, respectively, as the forward (analysis) and inverse (synthesis) Fourier representations for DAGs~\cite{seifert2023causal}. 
%
Through this lens, the $i$th column of the weighted transitive closure $\vecx w_i=\vecx W\vecx e_i$ 
is a Fourier mode of $\setx G$ (the DAG signal response
to a unit impulse cause at vertex $i$). The signal
$\vecx s=\sum_{i\in\setx V}c_i\vecx w_i$ is then synthesized from the root causes, which
play the role of frequency coefficients.
\vspace{2pt}
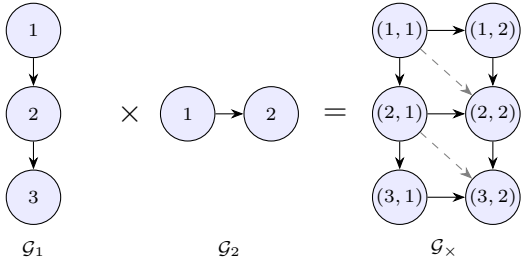
\begin{figure}[t]
\centering
\begin{tikzpicture}[>=Stealth,scale=0.85,every node/.style={font=\scriptsize},
   v/.style={circle,draw,fill=blue!8,minimum size=7.2mm,inner sep=0pt}]
  \node[v] (a1) at (0,0)    {$1$};
  \node[v] (a2) at (0,-1.3) {$2$};
  \node[v] (a3) at (0,-2.6) {$3$};
  \draw[->] (a1)--(a2); \draw[->] (a2)--(a3);
  \node at (0,-3.4) {$\setx G_1$};
  \node at (1.5,-1.3) {\large $\times$};
  \node[v] (b1) at (2.4,-1.3) {$1$};
  \node[v] (b2) at (3.7,-1.3) {$2$};
  \draw[->] (b1)--(b2);
  \node at (3.05,-3.4) {$\setx G_2$};
  \node at (4.7,-1.3) {\large $=$};
  \foreach \i in {1,2,3}{\foreach \j in {1,2}{
     \node[v] (p\i\j) at ({5.7+(\j-1)*1.45},{-(\i-1)*1.3}) {$(\i,\j)$};}}
  \draw[->] (p11)--(p21); \draw[->] (p21)--(p31);
  \draw[->] (p12)--(p22); \draw[->] (p22)--(p32);
  \draw[->] (p11)--(p12); \draw[->] (p21)--(p22); \draw[->] (p31)--(p32);
  \draw[->,dashed,gray] (p11)--(p22); \draw[->,dashed,gray] (p21)--(p32);
  \node at (6.4,-3.4) {$\setx G_\times$};
\end{tikzpicture}
\caption{\footnotesize A toy example of two factor DAGs $\setx G_1$ and $\setx G_2$ and their product $\setx G_\times$. The solid and dashed edges are, respectively, due to the first two terms and the last term of \eqref{eq:Atimes}.}
\label{fig:productgraph}
\end{figure}

\noindent\textbf{Causal Shift Operator.} The \emph{causal} shift indexed by $q\in \setx V$ is
built from the indicator $[\vecx d_q]_i:=\mathbb{I}\{i\preceq q\}$ of the principal
downset 
$\downset q:=\{i:i\preceq q\}$, i.e., $\vecx d_q\in\{0,1\}^{N}$
selects a set of frequencies, or root causes $\setx Q= \{i:[\vecx d_q]_i=1\}$. With $\vecx D_q:=\diag(\vecx d_q)$, we define the causal GSO
$\vecx T_q \in \R^{N \times N}$ on $\setx G$~\cite{seifert2023causal} as
\begin{equation}
\vecx T_q:=\vecx W\,\vecx D_q\,\vecx W^{-1}=\vecx W\,\vecx D_q\,(\vecx I-\vecx A).
\label{eq:posetshift}
\end{equation}
Since $\vecx T_q\vecx s=\vecx W\,\vecx D_q\,\vecx c$, the shift retains exactly the part of $\vecx s$ explained by the causes at or upstream of $q$. Causal shifts are poset analogues of time delays, where shifting ``to $q$'' accumulates all the causal contributions in $\downset q$.
 
Two remarks are in order. First, unlike $\vecx A$, the shift $\vecx T_q$ is {diagonalized by the Fourier basis}, as $\vecx W^{-1}\vecx T_q\vecx W=\vecx D_q$. Second, because $\vecx D_q$ is binary, $\vecx T_q$ is an \emph{idempotent} operator
\begin{equation}
\vecx T_q^2=\vecx W\vecx D_q^2\vecx W^{-1}=\vecx W\vecx D_q\vecx W^{-1}=\vecx T_q
\end{equation}
projecting onto $\operatorname{span}\{\vecx w_i:i\in \setx Q\}$ along the complementary modes. A single shift retains one band, whereas GSO combinations merge
several.\vspace{2pt}

\noindent\textbf{DAG Filters.} Given a \emph{frequency response}
$\tilde{\vecx h}=[\tilde{h}_1,\dots,\tilde{h}_N]^{\transp}\in\R^{N}$, we define the DAG filter~\cite{seifert2023causal} as the linear combination of shifts
\begin{equation}
\vecx H=\sum_{q\in \setx V}\!\!\theta_q\vecx T_q \!=\! \vecx W\diag(\tilde{\vecx h})\vecx W^{-1} \!= \!\vecx W \!\Big(\sum_{q\in \setx V}\!\theta_q\vecx D_q\!\Big)\!\vecx W^{-1}\!,
\label{eq:filtershift}
\end{equation}
where the filter coefficients $\theta_q$ satisfy $\tilde{\vecx h} = \sum_{q \in \setx V} \theta_q \vecx d_q =  \vecx Z \vecx \theta$. Here, $\vecx Z\!\in\!\{0,\!1\!\}^{N \!\times \!N}$ has entries
$Z_{iq}\!=\!\mathbb{I}\{i\!\preceq \!q\} = {[\vecx d_q]}_i$ and it is  triangular with unit diagonal, 
so it can be inverted to compute $\vecx \theta$. We say that a passband $\setx K \subseteq \setx V$ is \emph{admissible} \ac{iff} it is a downset, i.e., if $q \in \setx K$ and $i \preceq q$, then $i \in \setx K$. 

\section{Product DAGs} \label{sec:productdags}

Let $\setx G_1=(\setx V_1,\setx E_1,\vecx A_1)$ and $\setx G_2=(\setx V_2,\setx E_2,\vecx A_2)$ be DAGs with $N_1$ and $N_2$ vertices, respectively. 
A product DAG is the graph $\setx G_\times=(\setx V_\times,\setx E_\times,\vecx A_\times)$, where $\setx V_\times = \setx V_1 \times \setx V_2$. The entry $[\vecx A_\times]_{(i',j'),(i,j)}\neq 0$ exactly when $(i,j)$ is a parent of $(i',j')$ in the product graph.

The edge set $\setx E_\times$ and the resulting product DAG adjacency are determined by the choice of the product operator, say, the Kronecker product~\cite{sandryhaila2014bigdata}. Irrespective of that choice, $\vecx A_\times$ will be strictly lower triangular (modulo a topological sort) and hence nilpotent. In what follows, we propose a new DAG product operator that allows us to express the Fourier modes (respectively, the GSO) on the product DAG as the Kronecker product of the Fourier modes (causal shifts) of its DAG factors. Crucially, such factorizations are impossible using standard (e.g., Cartesian, Kronecker, or strong) graph products.

\begin{Definition}[Product DAG]\label{def:product}
The product of $\setx G_1$ and $\setx G_2$ is the DAG
$\setx G_\times=(\setx V_1\!\times\!\setx V_2,\setx E_\times,\Atimes)$ with adjacency matrix 
\begin{equation}
{\;\vecx A_\times :=\vecx A_1\otimes\vecx I_2+\vecx I_1\otimes\vecx A_2-\vecx A_1\otimes\vecx A_2\;}
\label{eq:Atimes}
\end{equation}
where $\vecx I_1\in\R^{N_1\times N_1}$ and $\vecx I_2\in\R^{N_2\times N_2}$ are the factor
identities.
\end{Definition}
Since $\vecx A_1$ and $\vecx A_2$ are strictly lower triangular, so is $\Atimes$, and $\setx G_\times$ is a DAG. Structurally, $\setx E_\times$ coincides with the edge set of the \emph{strong} product~\cite{sandryhaila2014bigdata}: vertex $(i',j')$ has parents of the form $(i,j')$, $(i',j)$, and $(i,j)$, where $i$ is a parent of $i'$ in $\setx G_1$ and $j$ is a parent of $j'$ in $\setx G_2$; see Fig.~\ref{fig:productgraph}.

What is new is the \emph{polarity} of the diagonal edges, which carry $-\vecx A_1\otimes\vecx A_2$ rather than $+\vecx A_1\otimes\vecx A_2$. As we show next, this sign flip is precisely what makes the resolvent, and hence the Fourier modes and causal shifts, factorize, as no standard graph product does. 

\begin{Proposition}[Resolvent factorization]\label{prop:resolvent}
With $\vecx A_\times$ as in \eqref{eq:Atimes}, it holds that
$\vecx I-\vecx A_\times=(\vecx I_1-\vecx A_1)\otimes(\vecx I_2-\vecx A_2)$. Hence
\begin{equation}
\vecx W_\times\!:=\!(\vecx I-\vecx A_\times)^{-1}\!=\vecx W_1\otimes\vecx W_2.
\label{eq:Wtimes}
\end{equation}
\end{Proposition}
\begin{proof}
Expanding the Kronecker product, we have $
(\vecx I_1-\vecx A_1)\otimes(\vecx I_2-\vecx A_2) = \vecx I_1\otimes\vecx I_2-\vecx A_1\otimes\vecx I_2-\vecx I_1\otimes\vecx A_2+\vecx A_1\otimes\vecx A_2 =\vecx I-\vecx A_\times$. 
Since each factor DAG is acyclic, $\vecx A_k$ is nilpotent and $\vecx I_k-\vecx A_k$ is
invertible, for $k\in\{1,2\}$. The Kronecker inverse identity
then gives
$\vecx W_\times=(\vecx I_1-\vecx A_1)^{-1}\otimes(\vecx I_2-\vecx A_2)^{-1}=\vecx W_1\otimes\vecx W_2$.
\end{proof}
There is more to Definition \ref{def:product} than the mathematical convenience of a factorizable resolvent. To further motivate the sign flip in \eqref{eq:Atimes}, let $i$ be a parent of $i'$ in $\setx G_1$ and $j$ a parent of $j'$ in $\setx G_2$. The Kronecker sum in \eqref{eq:Atimes} creates two directed paths of length $2$ between $(i,j)$ and $(i',j')$, each of them with a combined weight of $[\vecx A_1]_{i'i}[\vecx A_2]_{j'j}$, so that the information from  $(i,j)$ reaches $(i',j')$ twice. The negative cross term in \eqref{eq:Atimes} subtracts one of them, thus avoiding spurious duplication of information flowing from nodes' ancestors.

\subsection{Fourier Modes of the Product DAG}
Let the signal on the product DAG be a matrix $\vecx S\in\R^{N_1\times N_2}$ whose entry $S_{ij}$
is the nodal value at vertex $(i,j)$. Similarly, let $\vecx C \in \R^{N_1\times N_2}$ be the root cause matrix. Vectorizing them as $\vecx s=\vect(\vecx S^\transp)\in\R^{N}$ and $\vecx c=\vect(\vecx C^\transp)\in\R^{N}$ with $N = N_1N_2$, we arrive at
a linear \ac{sem} on the product DAG $\setx G_\times$, namely
\begin{equation}
    \vecx s=\Atimes\vecx s+\vecx c \quad \text{or, equivalently, } \quad \vecx s =\Wtimes \vecx c.
    \label{eq:prod-sem}
\end{equation}
From Proposition~\ref{prop:resolvent} and using $\vect(\vecx W_2\vecx C^\transp\vecx W_1^{\transp})=(\vecx W_1\otimes\vecx W_2)\vect(\vecx C^\transp)$, we have [cf. \eqref{eq:an_syn}]
\begin{equation}
\vecx S =\vecx W_1\,\vecx C\,\vecx W_2^{\transp} \quad\text{ and } \quad \vecx C=(\vecx I-\vecx A_1)\,\vecx S\,(\vecx I-\vecx A_2)^{\transp},
\end{equation}
where the signal is synthesized from the causes, and the causes are analyzed from the signal, through the factor transitive closures $\vecx W_1$ and $\vecx W_2$.

\subsection{Product DAG Causal Shift Operator} \label{sec:prodshift}

We now extend the causal GSO notion discussed in Section~\ref{sec:prelims} to product DAGs, wherein we leverage the favorable structure of $\Wtimes$ in \eqref{eq:Wtimes}. Since $\setx E_\times$ is the edge set of the strong product, reachability in $\setx G_\times$ is componentwise, i.e., $(i_1,i_2)\preceq(q_1,q_2)$ \ac{iff} $i_1\preceq_1 q_1$ and $i_2\preceq_2 q_2$. Hence the principal downset of a product vertex $q=(q_1,q_2)$ factorizes as $\downset(q_1,q_2)=\downset q_1\times\downset q_2$. Note that this is a purely structural property shared with the strong product, and the reason why the \emph{support} factorizes. The negative sign in \eqref{eq:Atimes} is what makes the \emph{weights} factorize. Thus, the causal shift indexed by $q  \in \setx V_\times$ is determined by the indicator 
$[\vecx d_{\times,q}]_{i_1i_2} = \mathbb{I}\{i_1\preceq_1 q_1,i_2\preceq_2 q_2\} = {[\vecx d_{q_1}^{(1)}]}_{i_1}{[\vecx d_{q_2}^{(2)}]}_{i_2}$. That is, $\vecx d_{\times,q} = \vecx d_{q_1}^{(1)} \otimes \vecx d_{q_2}^{(2)}$ and a product shift selects a set of product frequencies through a binary diagonal matrix $\vecx D_{\times,q} = \diag(\vecx d_{q_1}^{(1)} \otimes \vecx d_{q_2}^{(2)}) = \vecx D_{q_1}^{(1)} \otimes  \vecx D_{q_2}^{(2)} \in \{0,1\}^{N \times N}$, leading to the product DAG causal GSO [cf. \eqref{eq:posetshift}]
\begin{equation}
\begin{aligned}
    \vecx T_{\times,q}&=\vecx W_\times\vecx D_{\times,q}\vecx W_\times^{-1} 
    = \bigl(\vecx W_1\vecx D^{(1)}_{q_1}\vecx W_1^{-1}\bigr)\otimes \bigl(\vecx W_2\vecx D^{(2)}_{q_2}\vecx W_2^{-1}\bigr) \notag\\
&=\vecx T^{(1)}_{q_1} \otimes \vecx T^{(2)}_{q_2}.
    \label{eq:prodshift}
\end{aligned}
\end{equation}
In matrix form, the shifted signal $\vecx T_{\times,q}\vecx s$ reads $\vecx T^{(1)}_{q_1}\,\vecx S\,\bigl(\vecx T^{(2)}_{q_2}\bigr)^{\transp}
=\vecx W_1 \bigl(\vecx D^{(1)}_{q_1}\,\vecx C\,\vecx D^{(2)}_{q_2}\bigr)\vecx W_2^\transp$, which keeps the causes at or upstream of $q_1$ from $\setx G_1$ \emph{and} causes at or upstream of $q_2$ from $\setx G_2$, discarding the rest. For didactic purposes, we give a simple numerical example of the product DAG GSO for $\setx G_\times$ from Fig.~\ref{fig:productgraph}. 
\begin{worked}
\small
Consider the product of two chains $1\!\to\!2\!\to\!3$ and $1\!\to\!2$ in Fig.~\ref{fig:productgraph}, with weights $\frac{1}{2}$ on all the edges of the factor graphs. The transitive closure of the product is $\vecx W_\times=\vecx W_1\otimes\vecx W_2$ with
\[\setlength\arraycolsep{2.5pt}
\vecx W_1=\begin{bmatrix}1&0&0\\[-1pt] \frac12&1&0\\[0pt] \frac14&\frac12&1\end{bmatrix} \qquad \text{and} \qquad
\vecx W_2=\begin{bmatrix}1&0\\[-1pt] \frac12&1\end{bmatrix}.\]

For $q=(2,1)$, $\downset q=\{1,2\}\times\{1\}$ and
$\vecx T_{\times, (2,1)}=\vecx T^{(1)}_{2}\otimes\vecx T^{(2)}_{1}$ with
\[\setlength\arraycolsep{2.5pt}
\vecx T^{(1)}_{2}=\begin{bmatrix}1&0&0\\[-1pt] 0&1&0\\[-0pt] 0&\frac12&0\end{bmatrix} \qquad \text{and} \qquad
\vecx T^{(2)}_{1}=\begin{bmatrix}1&0\\[-0pt] \frac{1}{2}&0\end{bmatrix}.\]
One can readily verify that $\vecx T_{\times,(2,1)}$ keeps the root causes on $\downset q$. 
\end{worked}

\section{Product DAG Filters}\label{sec:filters}

In this section, we define product DAG filters and then focus on a specialized regime, namely, separable filters that act on each factor DAG independently. The DAG filters reviewed in Section~\ref{sec:prelims} can be extended to product DAGs by replacing the causal shift with the Kronecker causal GSOs from the factor DAGs. Let $\tilde{\vecx h} \in \R^N$ be the given frequency response, with entries $\tilde{h}_i, i \in \setx V_\times$ and $N = N_1N_2$. A product DAG filter is defined as a combination of the Kronecker causal shifts [cf. \eqref{eq:filtershift}]
\begin{equation*}
\vecx H_\times=\sum_{q_1\in \setx V_1} \sum_{q_2\in \setx V_2}\theta_{q_1q_2}
\vecx T^{(1)}_{q_1}\otimes \vecx T^{(2)}_{q_2}
=\Wtimes\diag(\tilde{\vecx h})\Wtimes^{-1}.
\end{equation*}
The filter coefficients $\theta_{q_1q_2}$ satisfy $\tilde{\vecx h} = \vect(\tilde{\vecx H}^\transp) = (\vecx Z_1 \otimes \vecx Z_2)\vect(\vecx \Theta^\transp)$, with ${[\vecx Z_k]}_{iq} = \mathbb{I}\{i \preceq_k q\} = {[\vecx d^{(k)}_{q}]}_i$ for $i,q \in \{1,\cdots, N_k\}$ and $k \in \{1,2\}$, where ${[\vecx \Theta]}_{q_1q_2} = \theta_{q_1q_2}$. Equivalently,
\[\tilde{\vecx H} = \vecx Z_1 \vecx \Theta \vecx Z_2^\transp \, \Leftrightarrow \,  \vecx \Theta = \vecx Z_1^{-1}\tilde{\vecx H} \vecx Z_2^{-\transp}.\] The filter $\vecx H_\times$ acts on a product DAG signal as $\hat{\vecx s} = \vecx H_\times{\vecx s}$ or, in matrix form, as
\begin{equation}
    \hat{\vecx S} = \vecx W_1(\tilde{\vecx H} \odot \vecx C) \vecx W_2^\transp,
    \label{eq:proddag_filter}
\end{equation}
where we recall $\vecx C=(\vecx I-\vecx A_1)\,\vecx S\,(\vecx I-\vecx A_2)^{\transp}.$ Since $\vecx Z_1$ and $\vecx Z_2$ are triangular with unit diagonal, any desired response $\tilde{\vecx H}$ is realized exactly.\vspace{2pt}

\noindent\textbf{Separable Product DAG Filters.} In the product DAG space, we say that a passband $\setx K \subseteq \setx V_1 \times \setx V_2$ is admissible \ac{iff} it is a downset of the product poset, i.e., if $(q_1,q_2) \in \setx K$, $i_1 \preceq_1 q_1$ and $i_2 \preceq_2 q_2$, then $(i_1,i_2) \in \setx K.$ Admissibility alone does not imply separability. A downset of the product poset need not factor as a product of downsets; see e.g., $\setx K = \{(1,1),(2,1),(1,2)\}$ for the product in Fig.~\ref{fig:productgraph}. Separability corresponds to the \emph{rectangular} subclass $\setx K = \setx K_1 \times \setx K_2$, where $\setx K_k$ is an admissible passband on $\setx G_k$, $k\in\{1,2\}$,  which is always admissible in the product. When the filter is built from such a passband, it has a separable structure with ${\vecx H}_\times = {\vecx H}_1 \otimes {\vecx H}_2$, where ${\vecx H}_k$ is the filter applied along factor $\setx G_k$. In this case, the filter coefficients factor across the two DAGs as $\vecx\Theta=\vecx\theta_1\vecx\theta_2^\transp$, hence $\rank(\vecx\Theta)=1$. The filter response decomposes as
\begin{equation}\label{eq:sep_response}
\tilde{\vecx H}=(\vecx Z_1\vecx\theta_1)(\vecx Z_2\vecx\theta_2)^\transp.
\end{equation}
Once more, to ground these concepts we offer a simple numerical illustration of a separable product DAG filter for $\setx G_\times$ in Fig.~\ref{fig:productgraph}. 
\begin{worked1}
\small
Let us revisit the product of two chains $1\!\to\!2\!\to\!3$ and $1\!\to\!2$ in Fig.~\ref{fig:productgraph}. The low-frequency bands $\setx K_1 = \{1,2\}$ and $\setx K_2 = \{1\}$ are admissible passbands on $\setx G_1$ and $\setx G_2$, respectively, and $\setx K = \setx K_1 \times \setx K_2 = \{(1,1), (2,1)\}$ is a rectangular admissible passband with outer-product response
\begin{equation*}
\tilde{\vecx H} =\vecx 1_{\setx K_1}\vecx 1_{\setx K_2}^\transp
=\begin{bmatrix}1\\1\\0\end{bmatrix}\!\begin{bmatrix}1&0\end{bmatrix}
=\begin{bmatrix}1&0\\1&0\\0&0\end{bmatrix},\qquad \rank(\tilde{\vecx H})=1.
\end{equation*}
One can readily verify that
\begin{equation*}
\tilde{\vecx H} =\vecx Z_{1} \vecx \Theta \vecx Z_2^\transp = \begin{bmatrix}1&1&1\\0&1&1\\0&0&1\end{bmatrix}\begin{bmatrix}0&0\\[-1pt] 1&0\\[0pt] 0&0\end{bmatrix} \begin{bmatrix}1&1\\0&1\end{bmatrix}^\transp.
\end{equation*}
The single nonzero entry $\Theta_{21}=1$ gives
$\vecx H_\times=\vecx T^{(1)}_{2}\otimes\vecx T^{(2)}_{1}=\vecx T_{\times,(2,1)}$, which realizes the low-pass band $\setx K=\{(1,1),(2,1)\}$ with the single product causal shift of Example 1.
\end{worked1}

\noindent\textbf{Filter Design.} A general target response $\tilde{\vecx H}$ is met exactly by $\vecx \Theta=\vecx Z_1^{-1}\tilde{\vecx H}\vecx Z_2^{-\transp}$, but a separable filter can reproduce it only when $\rank(\tilde{\vecx H})=1$. For general $\tilde{\vecx H}$, we thus fit the closest separable filter. Let $\vecx C=(\vecx I-\vecx A_1)\vecx S(\vecx I-\vecx A_2)^\transp$ be the analysis of the given signal and $\hat{\vecx S}=\vecx W_1(\tilde{\vecx H}\odot\vecx C)\vecx W_2^\transp$ be the target output in \eqref{eq:proddag_filter}. We design the separable filter \eqref{eq:sep_response} by minimizing the squared output deviation
\begin{equation*}
(\vecx\theta_1^\star,\vecx\theta_2^\star)=\operatorname*{arg\,min}_{\vecx\theta_1,\vecx\theta_2}
\big\|\,\hat{\vecx S}-\vecx W_1\big((\vecx Z_1\vecx\theta_1)(\vecx Z_2\vecx\theta_2)^\transp
\odot\vecx C\big)\vecx W_2^\transp\big\|_F^2, 
\end{equation*}
which can be solved using alternating least squares (ALS).

\section[Numerical Experiments]{Numerical Experiments\protect\footnotemark}
\label{sec:experiments}
\footnotetext{Code repository:  {\scriptsize \url{github.com/spchepuri/ProductDAG-denoising}}}

\begin{figure*}[t]
  \centering
  \begin{minipage}[c]{0.47\textwidth}
    \centering
    \includegraphics[width=\linewidth]{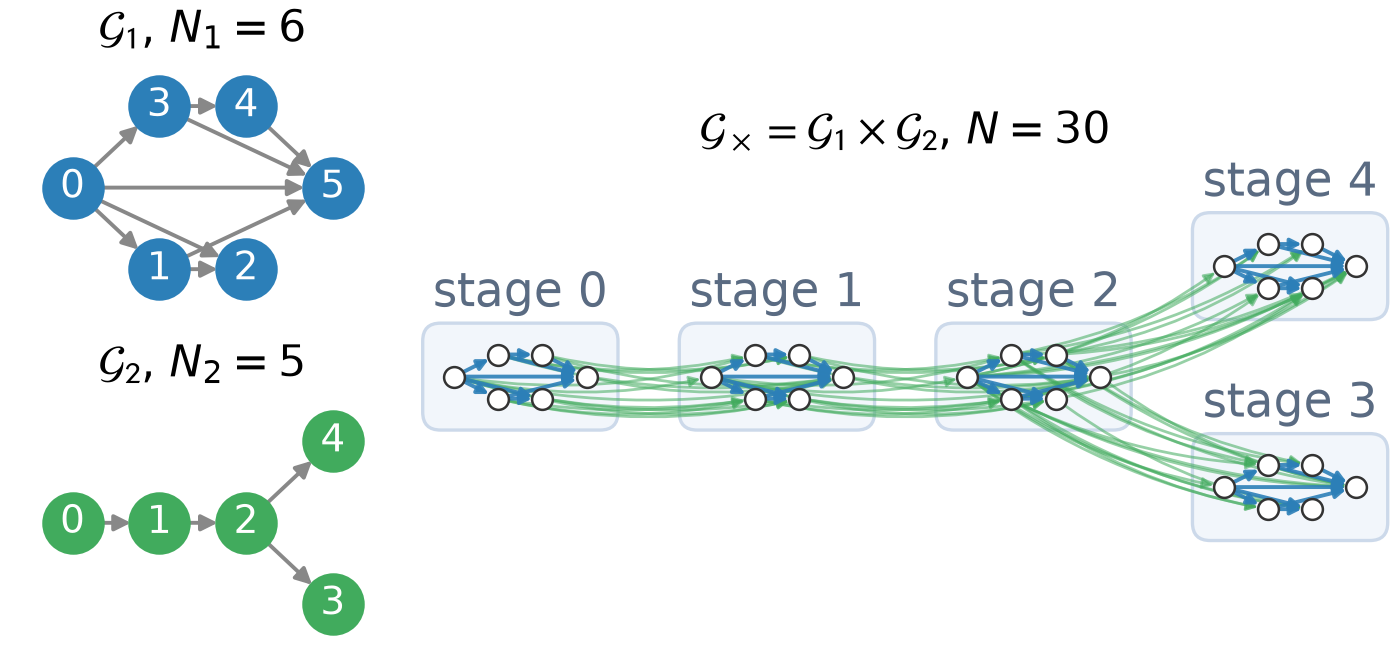}
    \par{\footnotesize (a) Factor DAGs and product graph}
  \end{minipage}~\hfill%
  \begin{minipage}[c]{0.51\textwidth}
    \centering
    \includegraphics[width=\linewidth]{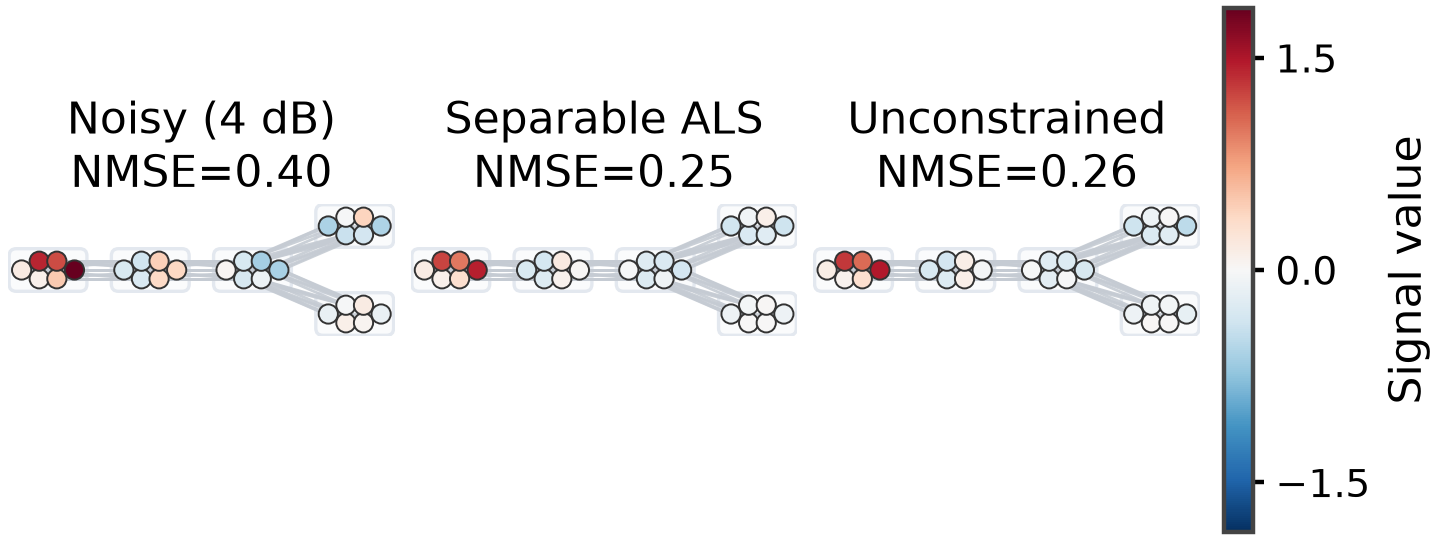}
    \par{\footnotesize (b) Denoising results}
  \end{minipage}
  \caption{\footnotesize Synthetic experiment:
    (a) Factor DAGs and their product;
    (b) Noisy signal, separable ALS denoising,
    and unconstrained denoising.}
  \label{fig:denoising}
  \vspace{-4mm}
\end{figure*}
\begin{figure}[t]
  \centering
  \includegraphics[width=0.75\linewidth]{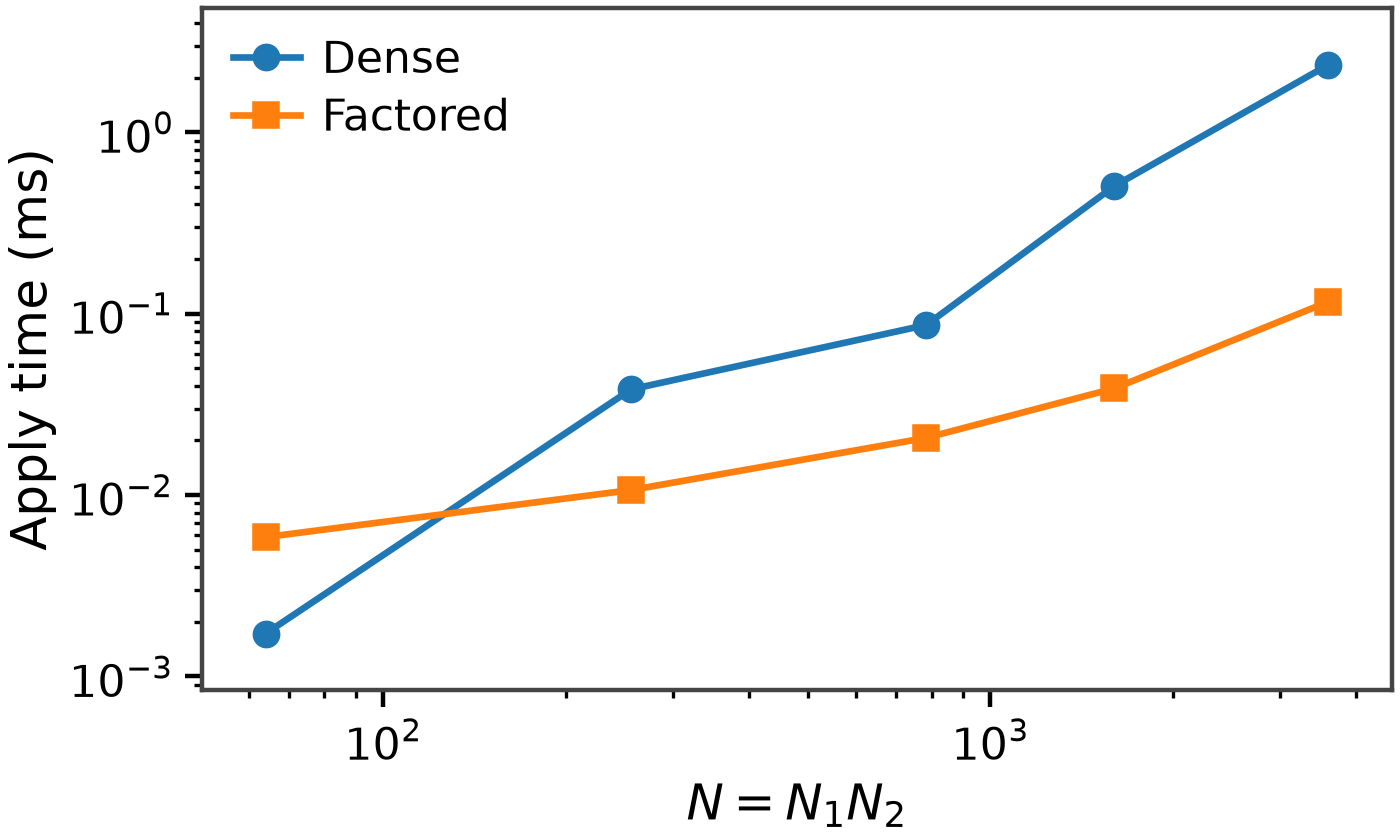}
  \caption{\footnotesize Filter application time for dense and product-factorized implementations versus graph size.}
  \label{fig:runtime}
  \vspace{-2mm}
\end{figure}

\noindent\textbf{Synthetic Dataset.} We evaluate the proposed product DAG filter design approaches on synthetic data generated from two random factor DAGs, $\setx G_1$ with $N_1=6$ nodes and $\setx G_2$ with $N_2=5$ nodes, giving a product DAG $\setx G_\times$ with $N=N_1N_2=30$ nodes; see Fig.~\ref{fig:denoising}(a). Each factor DAG is generated by adding a directed edge from node $u$ to node $v>u$ independently (with probability $0.45$ for $\setx G_1$ and $0.5$ for $\setx G_2$) and a weight drawn uniformly from $[0.3,0.7]$. Any resulting parentless node (except the first one) is connected to a predecessor chosen uniformly at random to guarantee weak connectivity. Clean signals $\matx S$ are generated by drawing independent zero-mean Gaussian coefficients $C_{ij}$ with $\mathbb{E}|C_{ij}|^2
= \rho_1^{\mathrm{depth}_1(i)}\,\rho_2^{\mathrm{depth}_2(j)} + 0.015$, where $\mathrm{depth}_k(i)$ is the longest-path distance from a source node to node $i$ in $\setx G_k$, $\rho_1 = 0.3$, and $\rho_2=0.34$. Thus, coefficient power decreases with graph depth, with a floor of $0.015$. The signal is then synthesized as $\matx S=\matx W_1\matx C\matx W_2^\transp$. Each clean signal is corrupted by white noise with an SNR of 4 dB. An example noisy signal is shown in Fig.~\ref{fig:denoising}(b-left). Filters are fit using 60 training (noisy, clean) pairs and evaluated on 80 held-out pairs. We compare the unconstrained response $\tilde{ \matx H}$ ($N_1N_2=30$ parameters), fitted by least squares, with rank-1 separable response $\tilde{ \matx H}=\vecx h_1\vecx h_2^\transp$ ($N_1+N_2=11$ parameters), fitted using the proposed ALS method. Fig.~\ref{fig:denoising}(b) shows held-out normalized mean-squared error (NMSE) of the two filtering methods relative to the unfiltered baseline. 
As can be seen, both the unconstrained fit (right) and the ALS method (center) are comparable, and significantly better than the unfiltered baseline (left). Fig.~\ref{fig:runtime} shows that the factorized product DAG filter, which avoids forming the dense $N \times N$ operator, is markedly faster than applying its dense equivalent once $N$ exceeds a few hundred nodes.\vspace{2pt}

\noindent\textbf{Real Dataset.} We evaluate the proposed filter on water-quality measurements from the CEH Thames Initiative~\cite{thamesdata}, comprising weekly observations at different sites along the River Thames and its tributaries from 2009 to 2023. We consider 52 weekly \emph{total dissolved phosphorus} observations per year at $N_1=15$ sites over eight years (2010–2017), where 94.1$\%$ of the site-week entries are directly measured and the remaining 5.9$\%$ are filled by linear interpolation within the same calendar year. We evaluate under an 8-fold leave-one-year-out cross-validation (CV) method, i.e., for each fold, one year is held out, and the filter is fit from the remaining seven years, each corrupted with six independently sampled white-noise realizations at a fixed SNR (yielding 42 noisy/clean training pairs). The held-out year is then corrupted with ten independent noise realizations, giving 80 held-out evaluations per method and SNR, with SNR $\in\{4,6,10\}$ dB. Fig.~\ref{fig:thames_results}(a) shows an example of data on the site DAG $\setx G_1$, and we use an $N_2=52$ node directed chain for $\setx G_2$.\vspace{2pt}

We recover the clean spatio-temporal field $\matx S \in \mathbb{R}^{15 \times 52}$ from a noisy observation $\matx Y = \matx S + \matx N$ using $\hat{\matx S}=\matx W_1(\tilde{\matx H}\odot \hat{\matx C})\matx W_2^\transp$, where $\matx N$ is additive white noise at a specified SNR and $\hat{\matx C} = \matx W_1^{-1} \matx Y \matx W_2^{-\transp}$. As before, we compare the unconstrained response $\tilde{\matx H}$ ($N_1N_2=780$ parameters), fitted by least squares, and the rank-1 separable response $\tilde{\matx H}=\vecx h_1\vecx h_2^\transp$ ($N_1+N_2=67$ parameters), fitted by ALS, against the no-filter baseline $\hat{\matx S}=\matx Y$. The table shown in Fig.~\ref{fig:thames_results}(b) reports the held-out NMSE. Both filters substantially outperform the baseline at every SNR value, with the separable filter consistently outperforming the unconstrained filter despite using far fewer parameters. For this denoising task and DAG, these results suggest that rank-1 separability provides an effective inductive bias for better generalization.

\begin{figure}[t]
  \centering
  \begin{minipage}[c]{0.54\linewidth}
    \centering
    \includegraphics[width=\linewidth]{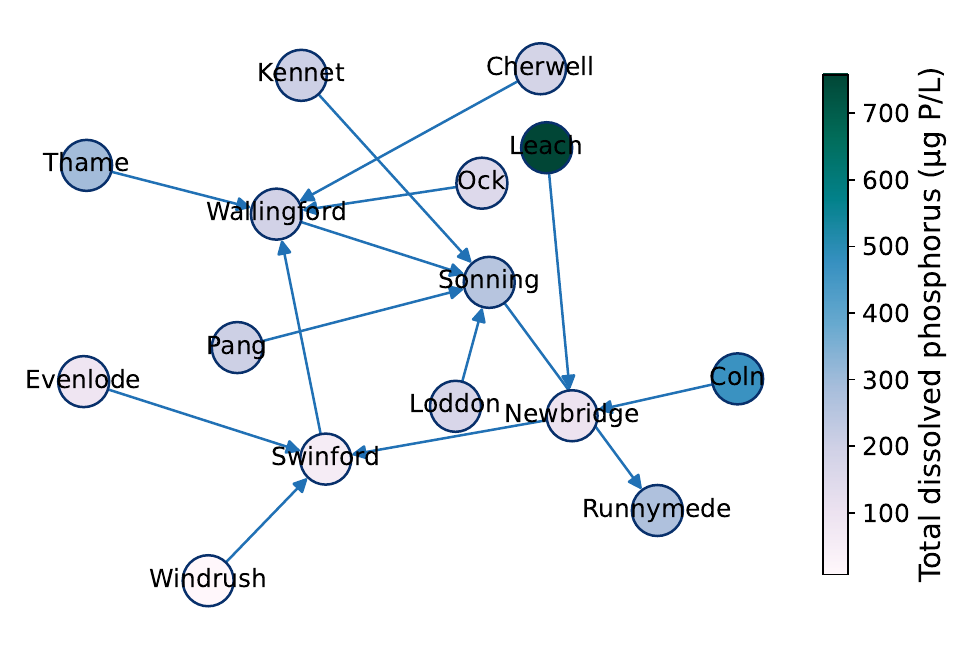}
    \par\smallskip
    {\footnotesize{(a)}}
  \end{minipage}
  \hfill
  \begin{minipage}[c]{0.44\linewidth}
  \centering
  \footnotesize
  \setlength{\tabcolsep}{1pt}
  \renewcommand{\arraystretch}{1.15}
  \begin{tabular}{@{}cccc@{}}
  \toprule
  \shortstack{SNR\\(dB)} &
  \shortstack{No\\filter} &
  ALS & Uncon. \\
  \midrule
  4  & 0.398 & \textbf{0.226} & 0.244 \\
  6  & 0.251 & \textbf{0.166} & 0.177 \\
  10 & 0.100 & \textbf{0.079} & 0.085 \\
  \bottomrule
  \vspace{12pt}
\end{tabular}
  \par\bigskip
  (b)
\end{minipage}
  \caption{\footnotesize River Thames experiment on total dissolved phosphorus data.
  (a) DAG of fifteen river sites. Noisy total dissolved phosphorus data from week 27 of 2017.
  (b) Held-out NMSE vs.\ SNR based on 8-fold leave-one-year-out CV. Bold: best per row. ``Uncon." stands for unconstrained filter design.}
  \label{fig:thames_results}
\end{figure}

\section{Conclusions}\label{sec:conclusions} 
We introduced a DAG product that preserves the two-axis causal structure and enables separable Fourier modes and filtering through resolvent factorization. Denoising experiments demonstrate the benefits of the resulting filters. Future work will explore filters beyond rank one and extend sampling and spectral estimation to product DAGs.

\newpage
\newpage
\section{Acknowledgements}
{\emergencystretch=2em 
Work supported by the KIAC Fintech grant, the US NSF under award ECCS 2231036, the Spanish AEI (AEI/\allowbreak10.13039/\allowbreak501100011033) grants PID2022-136887NB-I00 and PID2025-170000NB-I00, the Community of Madrid via IDEA-CM (TEC-2024/COM-89), and the ELLIS Madrid Unit. Beyond this support, the authors have no relevant financial or non-financial interests to disclose. Claude AI was used to assist in writing the manuscript and coding the simulations. The authors take full responsibility for the results of this paper.\par}

\section{Compliance with Ethical Standards}

This is a numerical simulation study for which no ethical approval was required.

\bibliographystyle{IEEEbib}
\bibliography{references}

\end{document}